\documentclass[conference]{IEEEtran}
\IEEEoverridecommandlockouts
\usepackage{cite}
\usepackage{amsmath,amssymb,amsfonts}
\usepackage{amsthm}
\usepackage[ruled,vlined, linesnumbered]{algorithm2e}
\usepackage{algpseudocode}
\usepackage{graphicx}
\usepackage{textcomp,hhline}
\usepackage{xcolor}

\def\BibTeX{{\rm B\kern-.05em{\sc i\kern-.025em b}\kern-.08em
    T\kern-.1667em\lower.7ex\hbox{E}\kern-.125emX}}

\newtheorem{assumption}{Assumption}
\newtheorem{theorem}{Theorem}

\renewcommand{\baselinestretch}{0.96}
\begin{document}

\title{Communication-Efficient Federated Learning with Dual-Side Low-Rank Compression}

\author{
    \IEEEauthorblockN{Zhefeng Qiao$^\star$, Xianghao Yu$^\star$, Jun Zhang$^\dagger$, and Khaled B. Letaief$^{\star \ddagger}$, \emph{Fellow, IEEE} }
    \IEEEauthorblockA{$^\star$Dept. of ECE, The Hong Kong University of Science and Technology, Hong Kong\\
    ${}^\dagger$Dept. of EIE, The Hong Kong Polytechnic University, Hong Kong\\
    ${}^\ddagger$Peng Cheng Laboratory, Shenzhen, China\\
    Email:{
        zqiaoaa@connect.ust.hk,
        eexyu@ust.hk,
        jun-eie.zhang@polyu.edu.hk,
        eekhaled@ust.hk
    }}
}

\maketitle

\begin{abstract}
Federated learning (FL) is a promising and powerful approach for training deep learning models without sharing the raw data of clients. During the training process of FL, the central server and distributed clients need to exchange a vast amount of model information periodically. To address the challenge of communication-intensive training, we propose a new training method, referred to as federated learning with dual-side low-rank compression (FedDLR), where the deep learning model is compressed via low-rank approximations at both the server and client sides. The proposed FedDLR not only reduces the communication overhead during the training stage but also directly generates a compact model to speed up the inference process. We shall provide convergence analysis, investigate the influence of the key parameters, and empirically show that FedDLR outperforms the state-of-the-art solutions in terms of both the communication and computation efficiency.

\end{abstract}


\section{Introduction}
The last several years have witnessed tremendous developments in deep learning, which revolutionized various applications, such as natural language processing, autonomous driving, and pattern recognition. The success of deep learning heavily relies on the availability of enormous training data samples that are collected in advance and stored in a centralized server. However, with the proliferation of smart mobile devices, massive data tend to be generated and stored locally. Furthermore, transmitting local data to a centralized server is not consistent with the increasing awareness of data privacy protection such as general data protection regulation (GDPR)\cite{voigt2017eu}. 

Federated learning (FL) \cite{mcmahan2017communication} is a promising solution to train a global deep learning model while keeping the private data locally. Instead of sharing the raw privacy-sensitive data, only model parameters are exchanged between the central server and clients. However, modern deep neural networks (DNNs) typically contain hundreds of millions of weight parameters\cite{simonyan2014very}. Therefore, to achieve a satisfactory training performance in FL, one must frequently transmit a large and complex model, which is clearly challenging, especially in communication-constrained application scenarios.

Recently, compression-based methods have been widely adopted to improve the communication efficiency of FL, where only a part of the weight or gradient information is transmitted. A dropout approach was considered in \cite{bouacida2020adaptive}, where a partial network is dropped during the training to reduce the number of parameters to be transmitted. Similarly, Jiang \textit{et. al.}\cite{jiang2019model} proposed an adaptive pruning method called PruneFL, which prunes the model during the training. 
In addition, gradient quantization was considered in \cite{konevcny2016federated,basu2019qsparse} to reduce the bit-width for each parameter to be transmitted. Top-$k$ sparsification was adopted in \cite{aji2017sparse}, which approximates the gradient matrix by its top $k$ entries. Nonetheless, these heuristic compression methods do not explicitly exploit the mathematical structure of the weight and gradient matrices, which may lead to a learning performance loss. Motivated by the fact that deep learning models typically have a low stable rank\cite{martin2018implicit}, low-rank compression was proposed for gradient and weight matrices in \cite{NEURIPS2019_d9fbed9d, martin2018implicit}. 
However, training with a single-side compression may generate a full-rank model for broadcasting at the server, which does not further reduce the communication cost in the downlink transmission and, more importantly, still maintains a bulky network for the inference stage.

In this paper, we propose a new training method for FL, which is referred to as federated learning with dual-side low-rank compression (FedDLR). In our proposed training method, once clients finish their local training, they perform a low-rank compression and then upload the respective model parameters to the central server. The central server then aggregates the received models into a global one and performs a low-rank compression of the aggregated model for broadcasting. By integrating the low-rank compression into FL to extract the principal components of the deep learning models at both sides of the server and clients, FedDLR is able to effectively reduce the dimensions of the models to be exchanged between the server and clients. Unlike existing works that investigated communication overhead reduction\cite{NEURIPS2019_d9fbed9d,zhou2020low} via client-side low-rank compression, the proposed FedDLR can improve both the communication efficiency in the training process and computation efficiency in the inference stage. Moreover, the convergence of the proposed FedDLR is analyzed. It is also proved that, thanks to the dual-side low-rank compression, the communication overhead is monotonically decreasing during the training process. To demonstrate the potential of FedDLR, we shall provide comprehensive experimental results. In particular, we will compare our proposal with two state-of-the-art methods, as well as, demonstrate the advantages of FedDLR in both the training and inference stages. Finally, we will investigate the impact of key parameters on our proposed FedDLR method.

The remainder of this paper is organized as follows. In Section \ref{SM}, we introduce the general FL system and the proposed FedDLR method. Section \ref{TAF} provides a theoretical analysis of FedDLR while experimental results are presented in Section \ref{Exper}. Finally, we conclude this work in Section \ref{Con}.

\section{System Model and Algorithm} \label{SM}
In this section, we first introduce the conventional FL system and the widely-adopted federated averaging algorithm (FedAvg) \cite{mcmahan2017communication}. We will then present the proposed communication-efficient FedDLR method. 

\subsection{Federated Learning System}
In this work, we focus on supervised FL for the ease of presentation. Assume that there are $K$ clients, each client $k$ has a local dataset $D_k$ and loss function $f^{k}$. Every client performs a local training with samples from its own dataset. The goal of the FL system is to find a global weight $\boldsymbol{w}$ which minimizes the global loss function $f$ given by
\begin{align}
    \min_{\boldsymbol{w}} f(\boldsymbol{w})= \frac{1}{\sum_{k=1}^{K} \lvert D_k \rvert} \sum^{K}_{k=1} \lvert D_k \rvert f^{k}(\boldsymbol{w}),
\end{align}
where $\lvert D_k \rvert $ represents the size of the local dataset $D_k$. Without loss of generality, we assume that the sizes of the local datasets are the same for all clients.

For such a complex distributed optimization problem, the optimal solution is typically achieved by variants of stochastic gradient descent (SGD). For example, FedAvg\cite{mcmahan2017communication} is one of the most popular federated learning algorithms, whose main update rules are illustrated in the following:
\begin{equation}
    \boldsymbol{w}^{k}_{t+1} =
      \begin{cases}
        \boldsymbol{w}^{k}_{t} - \eta_t \nabla f_{\boldsymbol{x}_t^k}(\boldsymbol{w}^{k}_t)& \text{if $(t+1) \vert R \neq 0$}\\
        \frac{1}{K}\sum_{k=1}^K \boldsymbol{w}^{k}_{t} - \eta_t \nabla f_{\boldsymbol{x}_t^k}(\boldsymbol{w}^{k}_t)& \text{if $(t+1) \vert R = 0$},
      \end{cases}  
\end{equation}
where $t$ is the iteration index, $\boldsymbol{x}_t^k$ are the data samples extracted from the local dataset $D_k$, $\boldsymbol{w}_t^{k}$ denotes the weight matrix at client $k$, $\eta_t$ is the learning rate, $\nabla f_{\boldsymbol{x}_t^k}(\boldsymbol{w}^{k}_t)$ is the stochastic gradient on $\boldsymbol{x}_t^k$, and $\cdot \vert \cdot$ represents the modulo operation. After every $R$ iterations, the clients upload the local model $\boldsymbol{w}_t^{k}$ to the central server. The central server then aggregates the received models, i.e., $\boldsymbol{w}_t=\frac{1}{K} \sum^{K}_{k=1} \boldsymbol{w}^k_t$, and broadcast the new global model $\boldsymbol{w}_t$ to all clients. This process is repeated until the model converges or reaches a desired performance, e.g., accuracy. 

Note that the communication between the clients and the central server only happens at the aggregation iterations, i.e., when $(t+1) \vert R = 0$. In addition, the communication overhead in iteration $t$ is proportional to the size of the model parameters, which can be extremely large in practice\cite{simonyan2014very}. To address this challenge, we propose a novel training method, namely, FedDLR for communication-efficient FL.

\subsection{Proposed FedDLR}
\begin{figure}[tp]
    \centerline{\includegraphics[width=0.5\textwidth , keepaspectratio]{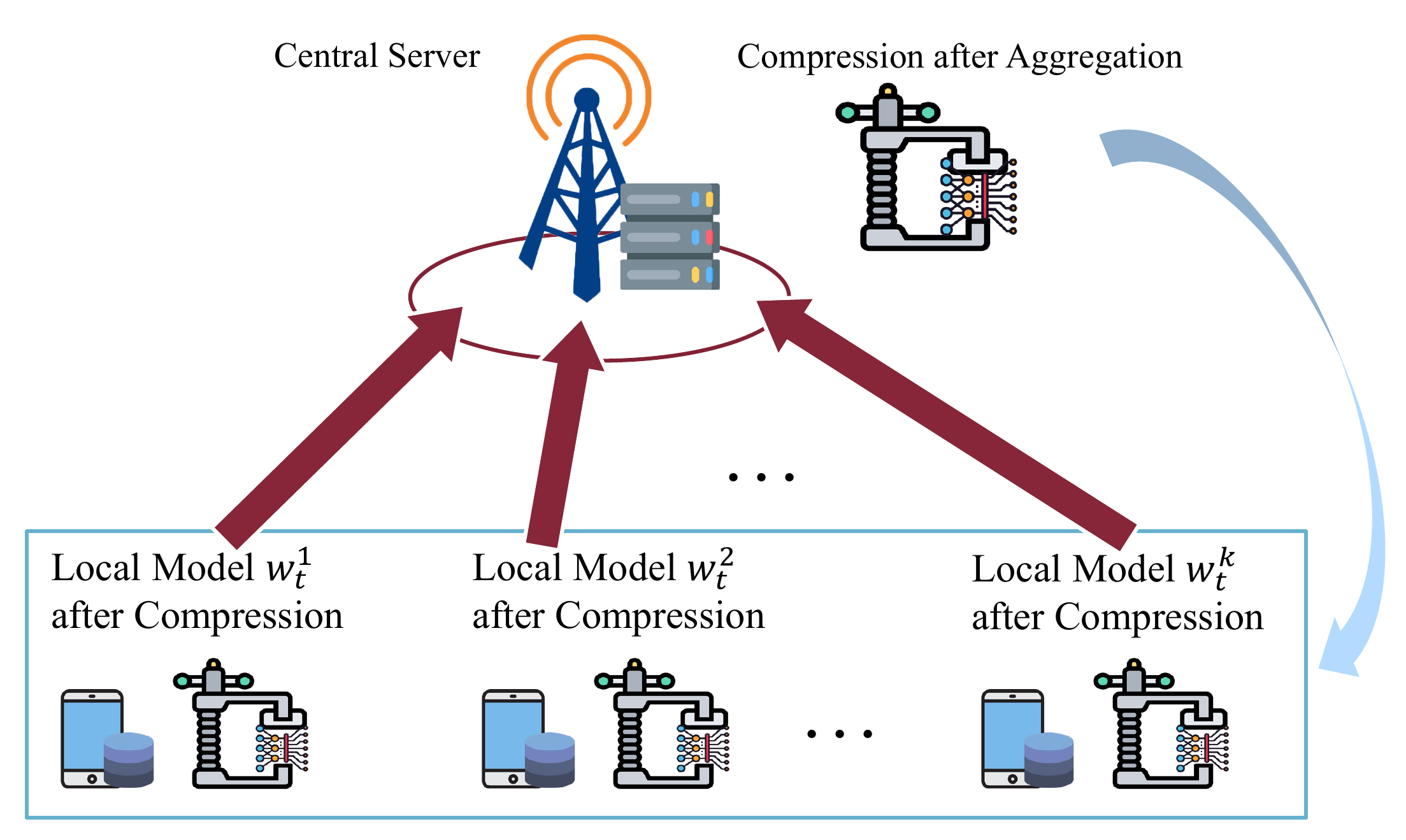}}
    \caption{Federated learning with dual-side compression.}
    \label{system_structure}
\end{figure}
To effectively reduce the communication overhead in FL, we propose a dual-side compression in each aggregation iteration. In particular, we perform compression before the clients upload the local model $\boldsymbol{w}_t^k$ and the server broadcasts the global model $\boldsymbol{w}_t^{k}$, respectively, as shown in Fig.  \ref{system_structure}. 
The corresponding training strategy is given by
\begin{equation}
    \boldsymbol{w}^{k}_{t+1} =
      \begin{cases}
        \boldsymbol{w}^{k}_{t} - \eta_t \nabla f_{\boldsymbol{x}_t^k}(\boldsymbol{w}^{k}_t) \qquad \text{if $(t+1) \vert R \neq 0$} &\\
        C_1\Big(\frac{1}{K}\sum_{k=1}^K C_2\Big(\boldsymbol{w}^{k}_{t} - \eta_t \nabla f_{\boldsymbol{x}_t^k}(\boldsymbol{w}^{k}_t)\Big)\Big)&\\
        \qquad \qquad \qquad \qquad \qquad \text{if $(t+1) \vert R = 0,$ }&
      \end{cases}    
      \label{FedDLR}   
\end{equation}
where $C_1(\cdot)$ and $C_2(\cdot)$ are the two compression functions before broadcasting at the sever and uploading at the clients, respectively.

It has been found that deep learning models are prone to be low-rank in many machine learning applications\cite{martin2018implicit}. Therefore, among various existing compression algorithms, we resort to low-rank compression for FL. Such kind of compression methods can also be considered as finding the optimal approximation of the weight matrices with rank constraints, known as the low-rank approximation. In this paper, we compress the model via the energy-based truncated singular value decomposition (TSVD)\cite{sadek2012svd}. Correspondingly, the two compression functions $C_1(\cdot)$ and $C_2(\cdot)$ in \eqref{FedDLR} are given by
\begin{align}\label{svd}
    C_l(\mathbf{W}) = \sum_{i=1}^{r} \sigma_i \mathbf{u}_i \mathbf{v}_i^T, \quad l \in \{1,2\},
\end{align}
where $\mathbf{W}$ is the  matrix to be compressed, $\sigma_i$ is the $i$-th largest singular value of $\mathbf{W}$ while $\mathbf{u}_i$ and $\mathbf{v}_i$ are the corresponding left and right singular vectors, respectively. Furthermore, $r$ denotes the desired compression rank. Intuitively, a smaller value of $r$ results in a lower communication overhead during the training of FL, which, however, maintains less information of the model. Hence, the desired rank $r$ for compression is a key parameter in FedDLR. In this paper, we adopt an energy-based criterion to dynamically adapt this value in each communication iteration. In particular, we determine $r$ as the smallest integer satisfying
\begin{align}
    \sum_{i=1}^{r} \sigma_i^2 \geq e  \sum_{i=1}^{\mathrm{Rank}(\boldsymbol{w})} \sigma_i^2,
    \label{e_TSVD}
\end{align}
where ${\mathrm{Rank}}(\cdot)$ denotes the rank of a matrix and $e \in (0,1]$ is a hyper-parameter representing the compression threshold. With a larger value of $e$, more principal components are kept and thus, the compression rank $r$ is larger. While for simplicity we assume that a common parameter $e$ is adopted at both the server and clients, it is straightforward to extend to the case where they are different.
The proposed FedDLR is summarized in \textbf{Algorithm \ref{FedDLR_Algo}}.

Next, we take the compression of the global model $\boldsymbol{w}_t$ as an example to illustrate the communication overhead reduction. For a given weight matrix $\boldsymbol{w}_t\in\mathbb{R}^{m \times n}$, there exists a low-rank approximation $\boldsymbol{w}_t \approx \mathbf{U}_t \times \mathbf{V}_t$, 
where $\mathbf{U}_t = [\sigma_{t1} \mathbf{u}_{t1}, \cdots, \sigma_{tr} \mathbf{u}_{tr}]$ and $\mathbf{V}_t = [\mathbf{v}_{t1}, \cdots, \mathbf{v}_{tr}]^T$ 
are full rank matrices with dimensions $m \times r$ and $r \times n$, respectively. As long as the desired compression rank $r$ is set to be small enough, i.e., $r < \frac{mn}{m+n}$, with the energy-based TSVD in  FedDLR, the number of model parameters to be transmitted between the clients and the server is reduced from $m \times n$ to $r \times (m+n)$. 

\emph{Remark 1:}
    Different from those communication reduction approaches only with client-side compression \cite{NEURIPS2019_d9fbed9d,zhou2020low}, we propose a dual-side low-rank compression training method in \eqref{FedDLR}, which is able to reduce  both the uplink and downlink communication overhead. 
    Moreover, the inference process of DNNs typically involves multiplication between matrices. Note that for $\mathbf{a} \in \mathbb{R}^{n}$, the computation complexity of the multiplication between $\boldsymbol{w}_t$ and $\mathbf{a}$ can be reduced from $\mathcal{O}(mn)$ to $\mathcal{O}(r(m+n))$ by applying the low-rank approximation to the multiplication $\mathbf{U}_t \mathbf{V}_t \mathbf{a}$. Hence, with the proposed FedDLR, we can improve the computation efficiency for inference, which is another inherent advantage of our proposed method. 
    More importantly, the communication cost  is non-increasing during the training stage, which we shall analytically show in the next section.

\begin{algorithm}[tp]
    \SetKwProg{SE}{Server Executes:}{}{}
    \SetKwProg{CE}{Clients Execute:}{}{}
    \let\oldnl\nl
    \newcommand{\nonl}{\renewcommand{\nl}{\let\nl\oldnl}}%
    \SetAlgoLined
    \SetKwProg{Func}{Function}{}{}
        Initialize all clients with parameter $\boldsymbol{w_0}$\\
        \CE{}{
        \For(){each client $k=1, \cdots, K$}{
            Download $\mathbf{U}_t, \mathbf{V}_t$ from the server \\
            $\boldsymbol{w}_t^k \gets \mathbf{U}_t \times \mathbf{V}_t$ \\
            \For(){each iteration $r = 0, \cdots, R-1$}{
                $\boldsymbol{w}_{t+r+1}^{k}$ $\gets \boldsymbol{w}^{k}_{t+r} - \eta_{t+r} \nabla f^k_{\boldsymbol{x}_{t+r}^k}(\boldsymbol{w}^{k}_{t+r})$\\
            }
            $\mathbf{U}^k_{t+R}, \mathbf{V}^k_{t+R} \gets$ \textit{LRCompression($\boldsymbol{w}_{t+R}^{k}$, e)}\\
            Upload $\mathbf{U}^k_{t+R}, \mathbf{V}^k_{t+R}$ to the server\\
        }
        }
        \nonl \DontPrintSemicolon \; 
        \SE{}{
        \For{$k=1, \cdots, K$}{
            Receive $\mathbf{U}^k_{t+R}, \mathbf{V}^k_{t+R}$ from client $k$ \\
            $\boldsymbol{w}^k_{t+R} \gets \mathbf{U}^k_{t+R} \times \mathbf{V}^k_{t+R}$
        }
        $\boldsymbol{w}_{t+R+1} \gets \frac{1}{K}\sum_{k=1}^{K} \boldsymbol{w}_{t+R}^k$\\
        $\mathbf{U}_{t+R+1}, \mathbf{V}_{t+R+1} \gets$ \textit{LRCompression($\boldsymbol{w}_{t+R+1}$, e)}\\
        Broadcast $\mathbf{U}_{t+R+1}, \mathbf{V}_{t+R+1}$ to clients \\
        }
        \nonl \DontPrintSemicolon \; 
        \Func{LRCompression($\boldsymbol{w}$, e)}{
            $\mathbf{u}_i, {\sigma}_i, \mathbf{v}_i \gets \textit{SVD}(\boldsymbol{w})$ according to \eqref{svd} \\
            $Total \gets 0$ \\
            $r \gets 0$\\
            \While(){True}{                
                $r \gets r+1$\\
                $Total \gets Total + \sigma_r^2$ \\
                \If(){$\frac{Total}{\sum_{i=1}^{\mathrm{Rank}(\boldsymbol{w})}{\sigma}_i^2} \geq e$}{
                    break
                }
            }
            Return $\mathbf{U}=[\sigma_1 \mathbf{u}_1, \cdots, \sigma_r \mathbf{u}_r]$ , $\mathbf{V} = [\mathbf{v}_1, \cdots, \mathbf{v}_r]^T$
        }
    \caption{Federated Learning with Dual-Side Low-Rank Compression (FedDLR)}
    \label{FedDLR_Algo}
\end{algorithm}

\section{Theoretical Analysis} \label{TAF}
In this section, we first provide a convergence analysis of the proposed FedDLR. We  then prove that during the local training and aggregation, the rank of the learning model is non-increasing, which is crucial for the communication overhead reduction.
\subsection{Convergence Analysis}
Before proving the convergence of the proposed FedDLR, we first present three key assumptions, based on which the theoretical results in this section are derived.
\begin{assumption}[Smoothness]
     The local loss function $f^{k}:\mathbb{R}^{m \times n} \rightarrow \mathbb{R}$ at each client $k$ is $L$-smooth, i.e., for  $ \boldsymbol{x}, \boldsymbol{y} \in \mathbb{R}^{m \times n}$, we have $f^{k}(\boldsymbol{y}) \leq f^{k}(\boldsymbol{x}) + \langle \nabla f^{k}(\boldsymbol{x}), \boldsymbol{y}-\boldsymbol{x} \rangle + \frac{L}{2}{\|\boldsymbol{y}-\boldsymbol{x}\|}^2$. 
     \label{smoothness_assum}
\end{assumption}
Following {Assumption \ref{smoothness_assum}}, it can be shown that the gradients are $L$-Lipchitz, i.e., $\|\nabla f^{k}(\boldsymbol{y}) - \nabla f^{k}(\boldsymbol{x})\| \leq L\|\boldsymbol{y}-\boldsymbol{x}\|$.
\begin{assumption}[Bounded Gradient]
    For $\boldsymbol{w}^{k}_t \in \mathbb{R}^{m \times n}$, the gradient is bounded by $ \mathop{\mathbb{E}}_{\boldsymbol{x}_t^k \sim D_k}[{\|\nabla f^k_{\boldsymbol{x}_t^k}(\boldsymbol{w}^{k}_t)\|}^2] \leq G_1^2$, where $G_1$ is a non-negative constant. 
    \label{gradient_assum}
\end{assumption}
Based on {Assumption 2}, we  have $ \mathop{\mathbb{E}}_{\boldsymbol{x}_t^k \sim D_k}[||\nabla f^k_{\boldsymbol{x}_t^k}(\boldsymbol{w}^{k}_t) - \nabla f^{k}(\boldsymbol{w}^{k}_t)||^2] \leq \delta^2$, where $\delta$ is a non-negative constant and $ \nabla f^{k}(\boldsymbol{w}^{k}_t) =\mathop{\mathbb{E}}_{\boldsymbol{x}_t^k \sim D_k}[\nabla f^k_{\boldsymbol{x}_t^k}(\boldsymbol{w}^{k}_t)]$.
We next introduce an auxiliary sequence  $\tilde{\boldsymbol{w}}^{k}_{t+1} = \boldsymbol{w}^{k}_{t} - \eta_{t} \nabla f^k_{\boldsymbol{x}^{k}_t}(\boldsymbol{w}^{k}_t)$ with $\tilde{\boldsymbol{w}}^{k}_0 = \boldsymbol{w}^{k}_0$ being the initial parameter.
\begin{assumption}[Bounded Weight]
    For $\tilde{\boldsymbol{w}}^{k}_t \in \mathbb{R}^{m\times n}$, we have $\|\tilde{\boldsymbol{w}}^{k}_t\|^2 \leq G_2^2$, where $G_2$ is a non-negative constant.
    \label{weight_assum}
\end{assumption}
{Assumptions \ref{smoothness_assum}} and {\ref{gradient_assum}} are common in the convergence analysis under the conventional FL setting\cite{basu2019qsparse, liu2021hierarchical}, which basically means that the variance of the gradients at all clients is bounded. {Assumption \ref{weight_assum}} is also a widely-used assumption for the convergence analysis of  machine learning models, e.g., recurrent neural networks \cite{pmlr-v108-chen20d}. Based on these assumptions, we have the following convergence guarantee for the proposed FedDLR.
\begin{theorem}
    Denote $T$ as the total number of iterations and $\mathcal{I}_T = \{t_0, t_1, \dots, t_h\} \in \{0, 1, \cdots, T-1\}$ as the set of aggregation indices, where $\lvert \mathcal{I}_T \rvert = \lceil \frac{T}{R} \rceil$ is the total number of aggregation steps. Likewise, define $H=\lvert \mathcal{I}_T \rvert $ and denote $b$ as the  mini-batch size, then we have
    \begin{align}
        \frac{1}{4KT}\sum_{t=0}^{T}\sum_{k=1}^{K} &\mathop{\mathbb{E}} \| \nabla f^k(\boldsymbol{w}^{k}_t)\|^2 \leq \frac{\mathop{\mathbb{E}}[f(\tilde{\boldsymbol{w}}_0)] - f^\star}{\eta_t T} + \frac{\eta_t L \delta^2}{b K}\nonumber\\
        &+ 2 \eta_t^2 L^2 G_1^2 R^2  
           + \frac{4(1-e^2)HL^2G^2_2}{T}.
        \label{theoremConvergenceEq}
    \end{align}  
    \label{theoremConvergence}
\end{theorem}
\vspace{-2em}
\begin{proof}
    Please refer to Appendix \ref{Proof of Theorem 1}.
\end{proof}
It is noted that in the right hand side of \eqref{theoremConvergenceEq}, the first term represents the distance between the initial point and the optimal solution. The second term is the error led by local SGD while the third term comes from the gradient deviations across all clients. With a decaying learning rate $\eta_t$, all of these three terms diminish to zero asymptotically. It is also noted that the last term represents the error introduced by the low-rank compression steps. When no compression is used, i.e., $e=1$, the convergence analysis reduces to the one for the conventional FedAvg case\cite{wang2019adaptive}.

\emph{Remark 2:}
    The last term in \eqref{theoremConvergenceEq} can be considered as a training error term controlled by $e$. It is resulted from the compression operation, which means that FL with low-rank compression may not converge to the optimal solution. However, since $1-e^2$ is monotonically decreasing over $e \in (0, 1]$, it follows that, when $e$ increases, i.e, a decrease in the compression level,  this error term  converges to zero asymptotically. In the next section, the impact of the hyper-parameter $e$ will be investigated via simulation.

\subsection{Communication Overhead}
As introduced in Section \ref{SM}, the communication overhead between each client and the central server is proportional to the desired rank $r$ for compression. However, since the energy-based criterion in (\ref{e_TSVD}) does not adjust the rank directly, it is unclear how the communication overhead changes during the training process, which is studied in the following theorem. First, following \cite{xu2020trp},  we extend {Assumption \ref{gradient_assum}} to the following assumption.
\begin{assumption}
    For  weight matrices $\boldsymbol{w}^{k}_t\in \mathbb{R}^{m \times n}$, we have $\max\{\|\tilde{\boldsymbol{w}}_{t_h}^k - \boldsymbol{w}^k_{t_{h-1}}\|,\|C_2(\tilde{\boldsymbol{w}}_{t_h}^k) - \boldsymbol{w}^k_{t_{h-1}}\|\} \leq G_3$, where $G_3$ satisfies $ G_3 \leq \sqrt{1-e}\min\{\|\tilde{\boldsymbol{w}}_{t_h}^{k}\|, \| \frac{1}{K}\sum_{k=1}^{K}C_2(\tilde{\boldsymbol{w}}_{t_h}^k)\|\}.$
    \label{extend_gradient}
\end{assumption}
Note that the norm of the weight difference is assumed to be bounded by that of the weight matrix in \cite{xu2020trp}, where centralized low-rank model training was investigated.
We extend this assumption to FL, where the norm of the weight difference is bounded by the minimum of the the weight norms before compression at the dual sides.
{Assumption \ref{extend_gradient}} shall be verified empirically in the next section. Based on this assumption, we obtain the following theorem to prove the monotonicity of the communication overhead during the training process. 
\begin{theorem}
    During the training process of FedDLR,  we have $\mathrm{Rank}(\tilde{\boldsymbol{w}}_{t_h}^k) \leq \mathrm{Rank}(C_1(\frac{1}{K}\sum_{k=1}^{K}C_2(\tilde{\boldsymbol{w}}_{t_{h-1}}^k)))$ and $\mathrm{Rank}(\frac{1}{K}\sum_{k=1}^{K}C_2(\tilde{\boldsymbol{w}}_{t_h}^k)) \leq \mathrm{Rank}(\boldsymbol{w}^k_{t_{h-1}})$. 
    \label{theoremRank}
\end{theorem}
\begin{proof}
    Please refer to Appendix \ref{Proof of Theorem 2}.
\end{proof}

Recall that the communication overhead is proportional to $r\times(m+n)$, and the two inequalities in Theorem \ref{theoremRank} indicate that the ranks of the models $r$ in both of the local training and aggregation steps do not increase. In contrast, this property cannot be guaranteed in existing client-side low-rank compression approaches since the global model aggregated in the server may not be low-rank. For instance, the method in \cite{zhou2020low} even increases the rank in the aggregation iterations.

\section{Experiments} \label{Exper}
In this section, we present simulation results to demonstrate the potential of FedDLR, we  also compare FedDLR with other state-of-the-art algorithms in terms of both the communication cost during the training stage and computation overhead during the inference stage. 
\subsection{Experimental Setting}
We consider an FL system with $K=10$ clients in total and each client has the same number of data samples for local training. We evaluate our proposed FedDLR method on the commonly-used image dataset CIFAR-10, which consists of 10 classes of $32\times 32$ color images. It contains 50,000 training examples and 10,000 testing examples. Thus, there are 5,000 training images in the local dataset at each client. We use the VGG-11 model\cite{simonyan2014very} with more than 9.7 million parameters for this classification task.

We use mini-batch SGD with batch-size $b=20$ as the optimizer, and set a decaying learning rate of $\eta_t = 0.1 \times 0.5^\frac{t}{10000}$ and compression threshold $e=0.990$. For local training, we use cross entropy as the loss function and perform 25 local training iterations. To show the effectiveness of the proposed FedDLR, we choose FedAvg\cite{mcmahan2017communication} and PruneFL\cite{jiang2019model} as the baselines to compare. 
\begin{figure*}[t]
	\centering
	\begin{minipage}[t]{0.24\linewidth} 
		\centering\includegraphics[width=1\textwidth]{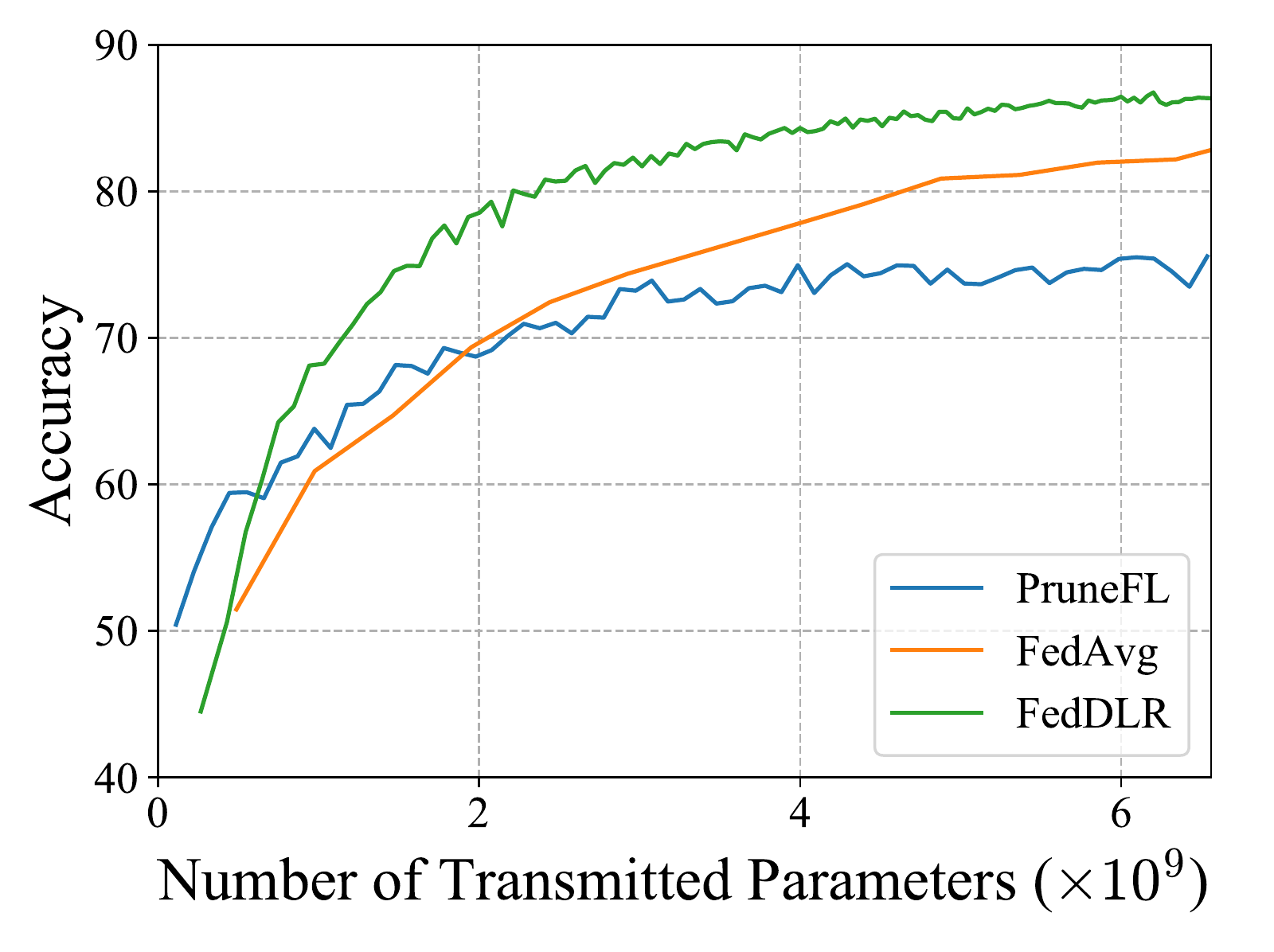}
		\caption{Accuracy versus the  communication cost.}
		\label{comm_acc_cifar10}
	\end{minipage}
	\begin{minipage}[t]{0.24\linewidth}
		\centering\includegraphics[width=1\textwidth]{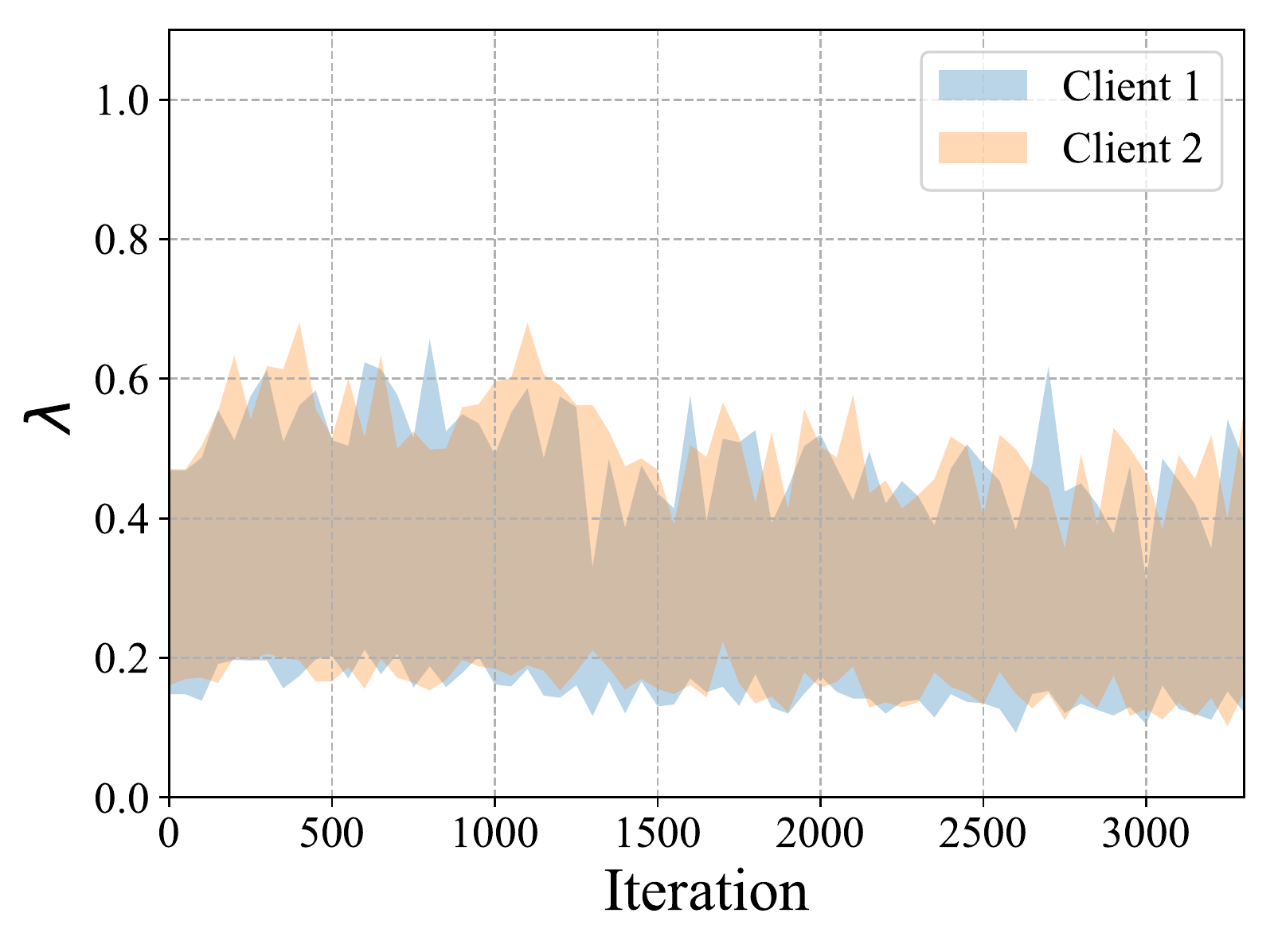}
		\caption{$\lambda$ versus the number of iterations.}
		\label{lambda}
	\end{minipage}
\begin{minipage}[t]{0.24\linewidth} 
	\centering\includegraphics[width=1\textwidth]{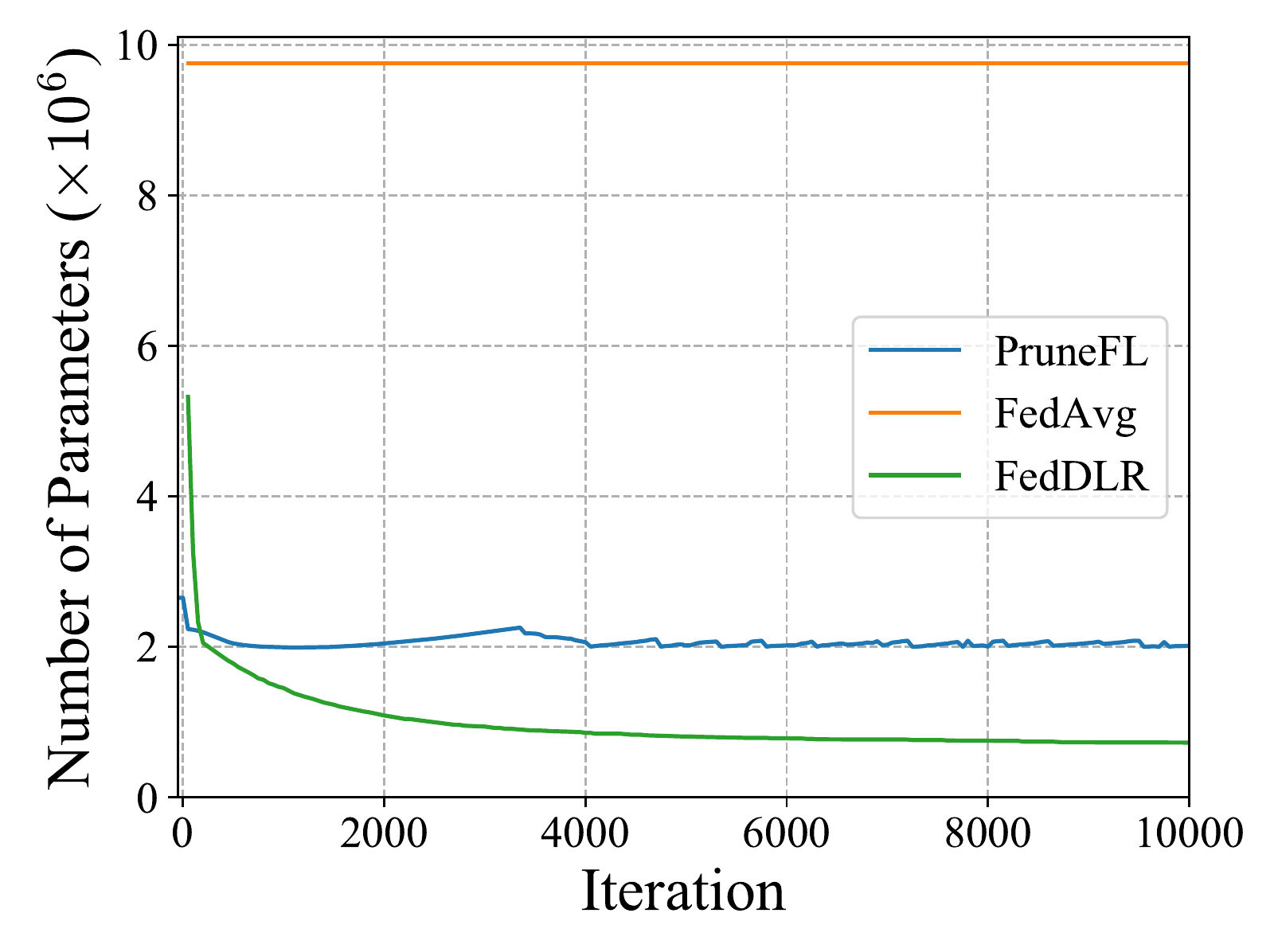}
	\caption{Number of parameters versus the number of iterations.}
	\label{round_params_cifar10}
\end{minipage}
\begin{minipage}[t]{0.245\linewidth}
	\centering{\includegraphics[width=1\textwidth]{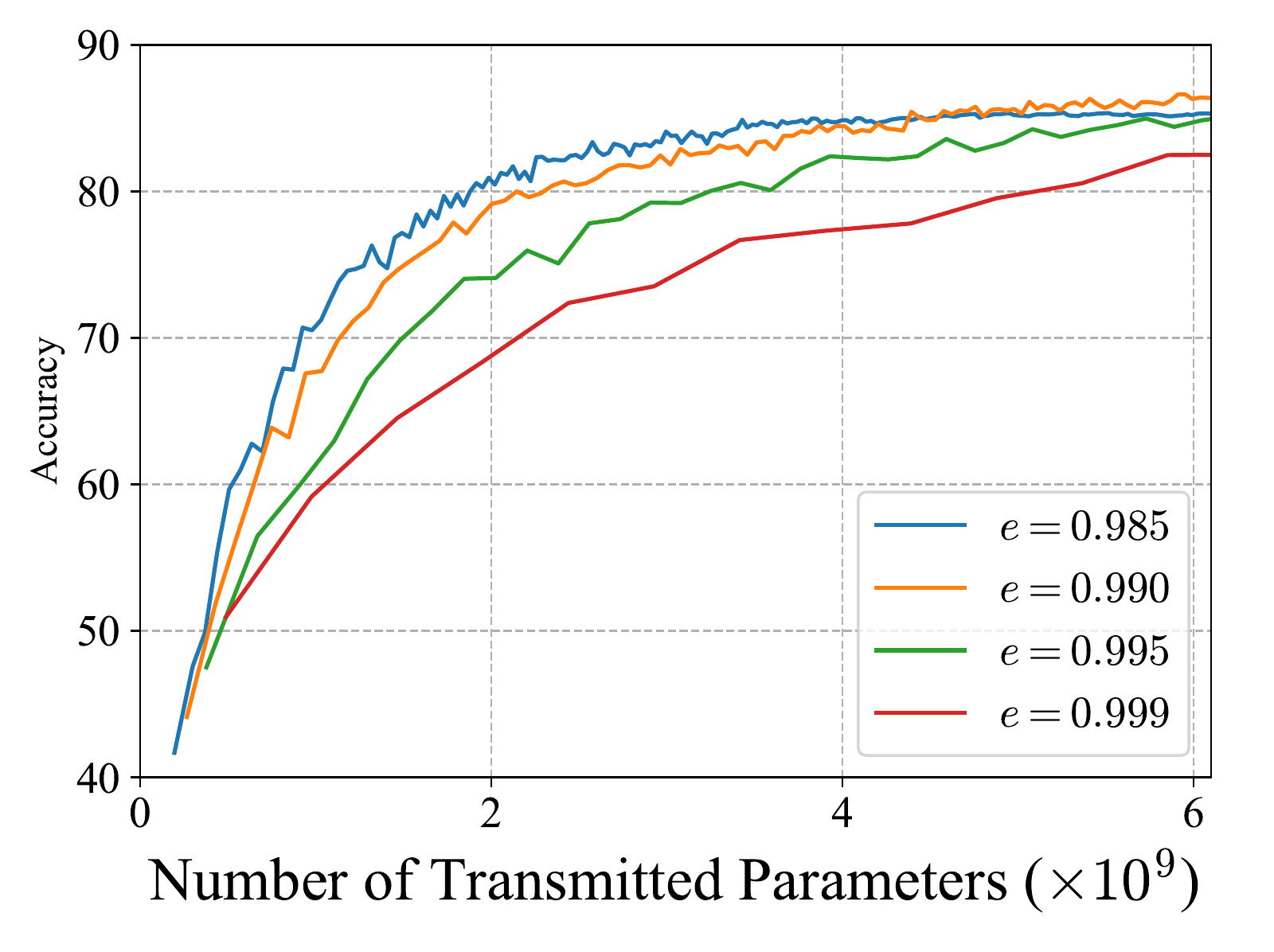}}
	\caption{Accuracy versus the  communication cost with different values of $e$.}
	\label{inf_e}
\end{minipage}
\end{figure*}
\subsection{Results}
We first evaluate the test accuracy versus the communication cost, defined by the number of transmitted model parameters during the training stage. As can be observed in Fig. \ref{comm_acc_cifar10}, given a fixed communication overhead of 6 billion transmitted parameters, the proposed FedDLR achieves a higher test accuracy (86.86\%) than  FedAvg (82.18\%) and PruneFL (75.56\%). 
Furthermore, FedDLR always achieves a better accuracy performance compared to the classic FedAvg during the whole training process, which clearly shows the advantage of introducing the model compression in FL.
Finally, PruneFL obtains a good accuracy at the very beginning during the training, while the convergence speed is much slower.
In contrast, the proposed FedDLR maintains a fast convergence speed among the three investigated FL training schemes. This phenomenon implies that, by exploiting the mathematical structure of the model matrices, the proposed low-rank compression preserves key model information more effectively.

In Figs. \ref{lambda} and \ref{round_params_cifar10}, we investigate the communication overhead during the training process. First, we verify {Assumption 4} by evaluating the values of $\lambda$\footnote{To keep the figure neat, we plot the envelopes of the values of $\lambda$ for all layers of the models at two clients. The values of $\lambda$ at other clients also satisfy $\lambda\leq1$.}, which is defined by $\lambda\triangleq\frac{\max\{\|\tilde{\boldsymbol{w}}_{t_h}^k - \boldsymbol{w}^k_{t_{h-1}}\|,\|C_2(\tilde{\boldsymbol{w}}_{t_h}^k) - \boldsymbol{w}^k_{t_{h-1}}\|\}}{\sqrt{1-e}\min\{\|\tilde{\boldsymbol{w}}_{t_h}^{k}\|, \| \frac{1}{K}\sum_{k=1}^{K}C_2(\tilde{\boldsymbol{w}}_{t_h}^k)\|\}}$. As can be observed in Fig. \ref{lambda},   $\lambda \leq 1$ always holds during the whole training process for all layers of our learning models, which confirms that  Assumption \ref{extend_gradient} is reasonable. 
Then, we evaluate the number of the model parameters that are transmitted in different FL training schemes.
As shown in Fig. \ref{round_params_cifar10}, the model size of FedAvg is constant and large. 
With a pruning operation, the model size in PruneFL is reduced compared to FedAvg. However, without theoretical guarantee, the number of transmitted parameters in PruneFL fluctuates around 2 million and can not be further reduced.
In contrast, our proposed FedDLR is able to monotonically decrease the communication overhead during the training, which verifies the theoretical results in Theorem 2. In addition, the number of parameters is reduced by half compared to the state-of-the-art PruneFL and quickly converges to 0.7 million, which in turn significantly reduces the communication cost in every aggregation step. 
\begin{table}[t]
    \caption{Computation Parameters.}
    \begin{center}
    \begin{tabular}{|c|c|c|c|}
        \hline
         & Parameters (Million) & MACs (Million) & Speedup \\\hhline{|=|=|=|=|}
        No Compression & 9.76 & 153.75 & $\times$1.00 \\ \hline
        FedDLR & \textbf{0.73} & \textbf{18.50} & $\times$\textbf{2.83}\\ 
        \hline
    \end{tabular}
    \label{infer_speed}
    \end{center}
\end{table}

In addition to the advantages in the training stage, our proposed FedDLR also benefits the inference process. 
Since our proposed method FedDLR generates a low-rank model, which approximates a large matrix with two much smaller matrices. Thus, it reduces the computation complexity to speedup the inference stage. 
Table \ref{infer_speed} lists the number of parameters, multiply-accumulate operations (MACs), and the inference efficiency of FedAvg and our proposed FedDLR. It is clearly shown that our compressed low-rank network significantly reduces the MACs and achieves a 2.83$\times$ speedup.

The impact of the hyper-parameter $e$ on the test accuracy is investigated in Fig. \ref{inf_e}. 
With a smaller value of $e$, the model converges faster, i.e., the number of transmitted parameters required for model convergence is smaller. This is because the smaller $e$ is, the fewer parameters are needed to be transmitted in each communication iteration.
However, it is worth noting that the accuracy of the  model first increases then decreases with the value of $e$.
In particular, when $e$ starts to decrease, more redundant model information is compressed. Therefore, the accuracy improves for smaller $e$ given the same amount of communication cost.
Nevertheless, when $e$ is exceedingly small, the overwhelming compression results in an accuracy drop. This also confirms the analytical results in Theorem 1 and Remark 2. Specifically, the error term in \eqref{theoremConvergenceEq} is a  monotonically deceasing function with respect to $e$, which leads to an accuracy loss when $e$ is too small. Hence, the hyper-parameter $e$ has to be carefully chosen in FedDLR to strike a balance between the communication cost, test accuracy, and convergence speed.

\section{Conclusion} \label{Con}
In this paper, we proposed a federated learning algorithm named FedDLR, which compresses the neural networks at both the local and central sides in FL to dramatically reduce the communication cost during the training while keeping satisfactory learning performance. Thanks to the low-rank property, the converged model learned by FedDLR also speedups the inference process. A convergence analysis of FedDLR was provided and we also proved that by FedDLR, the communication overhead during the training is non-increasing. Finally, based upon extensive experiments, it was shown that the proposed FedDLR outperforms other state-of-the-art methods in terms of both accuracy and learning efficiency.
\appendix
\renewcommand*{\thesection}{subsection}
\subsection{Proof of Theorem 1}\label{Proof of Theorem 1}

Similar to the analysis in \cite{basu2019qsparse}, we introduce some auxiliary sequences for each client $k$ and iteration $t$ as follows:
\begin{align}
    &\tilde{\boldsymbol{w}}^{k}_0 = \boldsymbol{w}^{k}_0 , \quad \tilde{\boldsymbol{w}}^{k}_{t+1} = \boldsymbol{w}^{k}_{t} - \eta_{t} \nabla f^k_{\boldsymbol{x}^{k}_t}(\boldsymbol{w}^{k}_t), \\
    &\boldsymbol{g}_t = \frac{1}{K}\sum_{k=1}^{K} \nabla f^k_{\boldsymbol{x}^{k}_t}(\boldsymbol{w}^{k}_t) ,\quad\bar{\boldsymbol{g}}_t = \frac{1}{K}\sum_{k=1}^{K}\nabla f^{k}(\boldsymbol{w}^{k}_t) ,\\
    &\boldsymbol{w}_t = \frac{1}{K}\sum_{k=1}^{K} \boldsymbol{w}^{k}_t,\quad \tilde{\boldsymbol{w}}_{t+1} = \frac{1}{K}\sum_{k=1}^{K} \tilde{\boldsymbol{w}}^{k}_{t+1} = \boldsymbol{w}_{t} - \eta_t \boldsymbol{g}_t.
\end{align}

Based on Assumptions \ref{smoothness_assum} and \ref{gradient_assum}, we  derive a general bound for the gradient expectation following the lines of the proof in \cite{basu2019qsparse} as follows
    \begin{align}\label{eq10}
        \frac{1}{4KT}\sum_{t=0}^{T}&\sum_{k=1}^{K} \mathop{\mathbb{E}} \| \nabla f^k(\boldsymbol{w}^{k}_t)\|^2 \leq \frac{\mathop{\mathbb{E}}[f(\tilde{\boldsymbol{w}}_0)] - f^\star}{\eta_t T} + \frac{\eta_t L \delta^2}{b K}\nonumber\\
        &+ 2 \eta_t^2 L^2 G_1^2 R^2   + \frac{2L^2}{T}\sum_{t=0}^{T}\mathop{\mathbb{E}}{\| \tilde{\boldsymbol{w}}_t - \boldsymbol{w}_t\|}^2.
    \end{align} 

    Note that, for the last term in the right hand side of \eqref{eq10}, we have
	\begin{equation}\label{aggeq}
	\tilde{\boldsymbol{w}}_t =\begin{cases}
	\boldsymbol{w}_t&\text{if } (t+1) \vert R \neq 0\\
	\boldsymbol{w}_{t-1} - \eta_{t-1} \boldsymbol{g}_{t-1}&\text{if } (t+1) \vert R = 0.\\
	\end{cases} 
	\end{equation}
    Then, we derive an upper bound for the last term in the right hand side of  \eqref{eq10}, given by
    \begin{align}
        &\frac{2L^2}{T}\sum_{t=0}^{T}\mathop{\mathbb{E}}{\| \tilde{\boldsymbol{w}}_t - \boldsymbol{w}_t\|}^2 \overset{(a)}{=} \frac{2L^2}{T} \sum_{h=0}^{H-1} \mathop{\mathbb{E}}{\| \tilde{\boldsymbol{w}}_{t_h} - \boldsymbol{w}_{t_h}\|}^2 \\
        &\overset{(b)}{\leq} \frac{4L^2}{T} \sum_{h=0}^{H-1} \mathop{\mathbb{E}}\bigg[\Big\|\tilde{\boldsymbol{w}}_{t_h} - \frac{1}{K} \sum_{k=1}^{K} C_2\big(\tilde{\boldsymbol{w}}_{t_h}^k\big) \Big\|^2\label{triangleeq}\\
        &+ \Big\|\frac{1}{K} \sum_{k=1}^{K} C_2\big(\tilde{\boldsymbol{w}}_{t_h}^k\big) - C_1\Big(\frac{1}{K} \sum_{k=1}^{K} C_2\big(\tilde{\boldsymbol{w}}_{t_h}^k )\big)\Big)\Big\|^2\bigg],\label{triangeleeq2}
    \end{align}
	where $(a)$ follows \eqref{aggeq} and step $(b)$ applies the triangle inequality.
	
    By using the inequality $\|\sum_{k=1}^{K} \boldsymbol{b}_k\|^2 \leq K\sum_{k=1}^{K}\|\boldsymbol{b}_k\|^2$ to bound the square-norm term in  \eqref{triangleeq}, we have
    \begin{align}\label{normeqbound1}
        \Big\|\tilde{\boldsymbol{w}}_{t_h} - \frac{1}{K} \sum_{k=1}^{K} C_2\big(\tilde{\boldsymbol{w}}_{t_h}^k\big) \Big\|^2 \leq& \frac{1}{K} \sum_{k=1}^{K} \Big\|\tilde{\boldsymbol{w}}_{t_h-1}^k - C_2\big(\tilde{\boldsymbol{w}}_{t_h-1}^k\big) \Big\|^2\nonumber\\
        \overset{(c)}{\leq} & (1-e)G_2^2,
    \end{align}
    where  $(c)$ is derived according to the energy-based TSVD method in \eqref{e_TSVD} and Assumption \ref{weight_assum}.

    Similarly, for the square-norm term in \eqref{triangeleeq2}, we have
    \begin{equation}\label{normeqbound2}
        \begin{split}
        &\Big\|\frac{1}{K} \sum_{k=1}^{K} C_2\big(\tilde{\boldsymbol{w}}_{t_h}^k\big) - C_1\Big(\frac{1}{K} \sum_{k=1}^{K} C_2\big(\tilde{\boldsymbol{w}}_{t_h}^k )\big)\Big)\Big\|^2 \\
         &\leq \frac{(1-e)}{K} \sum_{k=1}^{K}\|C_2(\tilde{\boldsymbol{w}}^{k}_{t_h})\|^2    \leq  (1-e)eG_2^2.
        \end{split}
    \end{equation}

    By substituting \eqref{normeqbound1} and \eqref{normeqbound2} into \eqref{triangleeq} and \eqref{triangeleeq2}, respectively, we complete the proof of Theorem \ref{theoremConvergence}.

\subsection{Proof of Theorem 2}\label{Proof of Theorem 2}

Suppose that the initial weight matrix for the local training  has rank $p$, i.e., $\mathrm{Rank}(\boldsymbol{w}^k_{t_{h-1}})= p$. Denote the singular values of $\tilde{\boldsymbol{w}}_{t_h}^k$ in a descending order as $\sigma^k_{t_h,1},\sigma^k_{t_h,2}, \cdots, \sigma^k_{t_h,s}$, where $s$ is the rank of $\tilde{\boldsymbol{w}}_{t_h}^k$. Following the proof in \cite{xu2020trp}, for the local training in each client, we have
\begin{align}
    \frac{\sum_{j=p+1}^{s} {\sigma^k_{t_h,j}}^2}{\sum_{j=1}^{s} {\sigma^k_{t_h,j}}^2} &\overset{(d)}{\leq} \frac{\|\tilde{\boldsymbol{w}}_{t_h}^k - \boldsymbol{w}^k_{t_{h-1}}\|^2}{\|\tilde{\boldsymbol{w}}^k_{t_{h}}\|^2} \leq 1-e
\end{align}
where $(d)$ follows the Eckart-Young-Mirsky theorem and the last step holds according to Assumption \ref{extend_gradient}.
It shows that the rank of the weight matrix does not increase in the local training with the energy-based TSVD.

Similarly, we prove the non-increasing property of the model rank at the server side in the following.
Suppose that the weight matrix broadcasted at the server in the last aggregation round has rank $q$, i.e., $\mathrm{Rank}(\boldsymbol{w}_{t_{h-1}})= q$.
Denote the singular values of $\frac{1}{K}\sum_{k=1}^{K}C_2(\tilde{\boldsymbol{w}}_{t_h}^k)$ in a descending order as $\sigma_{t_h1},\sigma_{t_h2}, \cdots, \sigma_{t_hz}$, where $z$ is the rank of $\frac{1}{K}\sum_{k=1}^{K}C_2(\tilde{\boldsymbol{w}}_{t_h}^k)$. Then, for the aggregation at the server, we have
\begin{equation}
    \begin{split}
    \frac{\sum_{j=q+1}^{z} \sigma_{t_hj}^2 }{\sum_{j=1}^{z} \sigma_{t_hj}^2 } &\leq \frac{\|\frac{1}{K}\sum_{k=1}^{K}C_2(\tilde{\boldsymbol{w}}_{t_h}^k) - \boldsymbol{w}_{t_{h-1}})\|^2}{\|\frac{1}{K}\sum_{k=1}^{K}C_2(\tilde{\boldsymbol{w}}_{t_h}^k)\|^2} \\
    &\leq  \frac{\frac{1}{K}\sum_{k=1}^{K} \|C_2(\tilde{\boldsymbol{w}}_{t_h}^k) - \boldsymbol{w}_{t_{h-1}}\|^2}{\|\frac{1}{K}\sum_{k=1}^{K}C_2(\tilde{\boldsymbol{w}}_{t_h}^k)\|^2} \\
    &\leq 1-e,
    \end{split}
\end{equation}
which completes the proof of Theorem 2.

\bibliographystyle{IEEEtran}
\bibliography{IEEEabrv,mybib}
\end{document}